\documentclass{ecai}

\usepackage{latexsym}

\usepackage{authblk}

\usepackage{times}  
\usepackage{helvet} 
\usepackage{courier}  
\usepackage[hyphens]{url}  
\usepackage{graphicx} 
\usepackage{graphicx}  

\usepackage{csquotes}

\usepackage{times}
\usepackage{graphicx}     
\usepackage{amsmath,amsthm}
\usepackage{txfonts}
\usepackage[ruled,vlined]{algorithm2e}

\usepackage{bbm}

\makeatletter
\newcommand{\newreptheorem}[2]{%
\newenvironment{rep#1}[1]{%
 \def\rep@title{#2 \ref{##1}}%
 \begin{rep@theorem}}%
 {\end{rep@theorem}}}
\makeatother

\newcommand{\pazocal}{\cal}

\newcommand\independent{\protect\mathpalette{\protect\independenT}{\perp}}
\def\independenT#1#2{\mathrel{\rlap{$#1#2$}\mkern2mu{#1#2}}}
\newtheorem{THEOREM}{Theorem}
\renewenvironment{theorem}{\begin{THEOREM} }%
                        {\end{THEOREM}}

\newtheorem{LEMMA}[THEOREM]{Lemma}
\newenvironment{lemma}{\begin{LEMMA} \hspace{-.85em} {\bf :} }%
                      {\end{LEMMA}}
\newtheorem{COROLLARY}[THEOREM]{Corollary}
                          {\end{COROLLARY}}
\newtheorem{PROPOSITION}[THEOREM]{Proposition}
\newenvironment{proposition}{\begin{PROPOSITION} \hspace{-.85em} {\bf :} }%
                            {\end{PROPOSITION}}
\newtheorem{DEFINITION}[THEOREM]{Definition}
\newenvironment{definition}{\begin{DEFINITION} \rm}%
                            {\end{DEFINITION}}
\newtheorem{CLAIM}[THEOREM]{Claim}
                            {\end{CLAIM}}
\newtheorem{EXAMPLE}[THEOREM]{Example}
                            {\end{EXAMPLE}}
\newtheorem{REMARK}[THEOREM]{Remark}
                            {\end{REMARK}}
							\newtheorem{NOTATION}[THEOREM]{Notation}
							                            {\end{NOTATION}}
\renewenvironment{proof}{\noindent {\bf Proof:} \hspace{.677em}}%
                     {}

\DeclareMathAlphabet{\mathitbf}{OML}{cmm}{b}{it}

\renewcommand{\S}{{\cal S}}

\newcommand{\U}{{\bf U}}
\newcommand{\V}{{\bf V}}
\newcommand{\W}{{\bf W}}
\newcommand{\X}{{\bf X}}
\newcommand{\Y}{{\bf Y}}
\newcommand{\Z}{{\bf Z}}

\newcommand{\set}[1]{\left\{ #1 \right\}}

\newcommand{\blemma}{\begin{lemma}}

\newcommand{\elemma}{\end{lemma}}

\newcommand{\bthm}{\begin{theorem}}
\newcommand{\ethm}{\end{theorem}}
\newcommand{\bprf}{\begin{proof}}
\newcommand{\eprf}{\end{proof}}
\newcommand{\bpro}{\begin{proposition}}
\newcommand{\epro}{\end{proposition}}
\newcommand{\bi}{\begin{itemize}}
\newcommand{\ei}{\end{itemize}}
\newcommand{\be}{\begin{enumerate}}
\newcommand{\ee}{\end{enumerate}}
\newcommand{\beq}{\begin{equation}}
\newcommand{\eeq}{\end{equation}}
\newcommand{\bcase}{\begin{cases}}
\newcommand{\ecase}{\end{cases}}
\renewcommand{\mathit}{\emph}

\begin{document}


	
\title{Interventions and Counterfactuals \\ 
in Tractable Probabilistic Models}

\title{On the Causal Expressiveness of \\   Tractable Probabilistic Models}

\title{Interventions and Counterfactuals in Tractable Probabilistic Models:\\ Limitations of Contemporary Transformations}

\author{Ioannis Papantonis\institute{University of Edinburgh, \\ email:\tt~ i.papantonis@sms.ed.ac.uk} \and Vaishak Belle\institute{University of Edinburgh \& Alan Turing Institute, \\ email:\tt~ vaishak@ed.ac.uk}}


%
\maketitle

\begin{abstract} In recent years, there has been an increasing interest in studying causality-related properties in machine learning models generally, and in generative models in particular. While that is well motivated, it inherits the fundamental computational hardness of probabilistic inference, making exact reasoning intractable. Probabilistic tractable models have also recently emerged, which guarantee that conditional marginals can be computed in time linear in the size of the model, where the model is usually learned from data. Although  initially limited to low tree-width models, recent tractable models such as sum product networks (SPNs) and probabilistic sentential decision diagrams (PSDDs) 
exploit efficient function representations and also capture high tree-width models.  

In this paper, we ask the following technical question: can we use the distributions represented or learned by these models to perform causal queries, such as reasoning about interventions and counterfactuals? By appealing to some existing ideas on transforming such models to Bayesian networks, we answer mostly in the negative. We show that when transforming SPNs to a causal graph interventional reasoning reduces to computing marginal distributions; in other words, only trivial causal reasoning is possible. For PSDDs the situation is only slightly better. We first provide an algorithm for constructing a causal graph from a PSDD, which introduces augmented variables. Intervening on the original variables, once again, reduces to marginal distributions, but when intervening on the augmented variables, a deterministic but nonetheless \enquote{causal-semantics} can be provided for PSDDs.

\end{abstract}

\section{Introduction}

 In recent years, there has been an increasing interest in studying causality-related properties in machine learning models. For example, \cite{kansky2017schema} argue for the ability to assess past observations and explain away alternative causes in deep reinforcement learning methods. In \cite{DBLP:journals/corr/abs-1811-10597},  the question of what units are responsible for controlling and manipulating certain features within an image is considered. In \cite{DBLP:journals/corr/abs-1812-03253}, strategies to give a causal interpretation to the intrinsic structure of deep learning models is investigated. Broadly speaking \cite{pearl2019seven}, the motivation stems from extending the query and reasoning capabilities over probabilistic domains. That is, in standard probabilistic models, one is simply interested in \emph{conditioning on observations} \( 
 \Pr(y
 \mid x) \): e.g., what is the likelihood of lung inflammations given that the patient smokes? 
 %
 Causal reasoning allows us to reason about \emph{interventions} \( 
 \Pr(y \mid do(x)) \): e.g., how are  lung inflammations affected when the patient reduces the amount of tobacco   smoked  in a day? \emph{Counterfactual queries} allow us to directly reason about alternate  worlds \( 
 \Pr(y \mid x^ *) \): e.g., what state would the patient's lung inflammations be in had   he not smoked in the previous year? Thus, causal reasoning allows us to inspect our domain model much more comprehensively than possible by observational conditioning alone.

A fundamental challenge underlying stochastic models, however, is the intractability of inference \cite{article}. This has led to the paradigm of \emph{tractable probabilistic models}, where conditional or marginal distributions can be computed in time linear in the size of the model. Although  initially limited to low tree-width models \cite{bach2002thin}, recent tractable models such as sum product networks (SPNs)  \cite{6130310,SPN_structure_learning} and probabilistic sentential decision diagrams (PSDDs) \cite{kisa2014probabilistic} are derived from arithmetic circuits (ACs) and knowledge compilation  approaches, more generally  \cite{darwiche2002logical,choi2017relaxing}, which  exploit efficient function representations and also capture high tree-width models. These models can also be learnt from data \cite{Peharz2014LearningSS,kisa2014probabilistic} which leverage the efficiency of inference. Consider that in classical  structure learning approaches for graphical models, once learned, inference would have to be approximated, owing to its intractability. In that regard, such models offer a robust and tractable framework for learning and inferring from data.

Naturally, owing to such attractive properties, the theoretical underpinnings of such models have received considerable attention. One the one hand, when viewed from a knowledge representation angle, they are related to tractable representations of Boolean functions, including BDDs \cite{Akers:1978:BDD:1310167.1310815} and d-DNNFs \cite{ddnnf,kisa2014probabilistic}. On the other hand, from a probabilistic modeling perspective, they can be derived as instances of ACs, providing a tractable representation of probabilistic reasoning, owing to the fact that ACs can compactly represent and compute the network polynomial of a Bayesian network (BN)  \cite{Darwiche2000ADA}. In the presence of latent variables they can also be seen as a deep architecture with probabilistic semantics \cite{6130310}, leading to numerous extensions, e.g.,  for mixed discrete-continuous domains \cite{inproceedings18}, and applications, including preference ranking \cite{Choi2015TractableLF}, classification \cite{Liang2018LearningLC} and computer vision \cite{6130310,article17}. Owing to its clear probabilistic semantics, in \cite{NIPS2011_4350}, the expressive power of such deep models is studied, and in \cite{Zhao:2015:RSN:3045118.3045132}, the relationship between SPNs and BNs has been further analyzed.

In this paper we push the envelope on this inquiry towards the following objective: can tractable models offer not only a computationally attractive but also compelling alternative to standard graphical models, especially 
when reasoning about causality? In fact, on studying  the relationship between SPNs and BNs  \cite{Zhao:2015:RSN:3045118.3045132}, the authors conclude with: 
\begin{quote}\it
The structure of the resulting BNs can be 
used to study probabilistic dependencies 
and causal relationships between the
variables of the original SPNs.   
\end{quote}

Unfortunately, we answer in the negative for SPNs and PSDDs. 

For SPNs, using the transformation from \cite{Zhao:2015:RSN:3045118.3045132} we are going to show that the resulting graph is not sufficient for studying causal relationships between the variables. Roughly, the problem is that this class of models allows for a lot of expressive freedom, and, because of that, all the correlations between the variables are attributed to external latent factors. Next, for PSDDs, we first provide an algorithm for constructing a causal graph that also needs to introduce some augmented variables, which conforms to PSDD on all probabilistic queries. On the one hand, intervening on the original variables in the resulting causal graph is also uninteresting, and reduces to computing marginal distributions, like in  SPNs. However we can perform non-trivial counterfactuals on the augmented variables. This is possible because, in contrast to SPNs, PSDDs impose more restrictions on the structure of the resulting model, specifically in terms of its equivalence to a propositional formula, which we then can use  to recover a structural equation model (SEM) \cite{Pearl2009CausalII}. We note that this structure is of a somewhat \enquote{deterministic} nature, and so, in a sense, the result is also negative. Nonetheless, we can provide a ``causal semantics" for PSDDs in the process.

\textit{We reiterate that our focus is purely on the distributions represented or learnt using tractable probabilistic models, and specifically SPNs and PSDDs.} These models do not come with any guarantees that the dependencies learnt actually capture the underlying causal process of the domain (in contrast to approaches such as \cite{DBLP:conf/icml/GhassamiSKB18}). \textit{Throughout the rest of our analysis, we suppose that the causal graph is not known beforehand and our aim is to examine what kind of information can be recovered using the trained SPNs or PSDDs. Our results demonstrate that regardless of whether these models do capture causal information, performing causal reasoning on them is not immediately attractive.}

We are organized as follows. We first consider the SPN case, and then move on to the PSDD case. We will finally conclude with some discussions.

\section{Background}

We will briefly review PSDDs, SPNs and SEMs,  and we refer the reader to \cite{kisa2014probabilistic,6130310,10.2307/3541871} for more extensive discussions.  

Our notation will be as follows: An uppercase letter \( X \) denotes a Boolean random variable. In the context of a probabilistic statement, we will use a lowercase letter \( x \) to denote an assignment  to \( X \); for example, \( \Pr(X=x) \) where \( x\in \set{0,1} \) denotes the probability of the event where \( X \) is assigned the value \( x. \) In the context of a logical formula, $X$ and $\neg X$ respectively assign \emph{true} ($\top$) and \emph{false} ($\bot$) to variable $X$. Sets of variables $\X$ and joint assignments $\bf x$ are denoted in {\bf bold}. \smallskip

{\bf PSDDs.} ~~ 
The idea behind PSDDs is to use Sentential Decision Diagrams (SDDs) \cite{inproceedings} to represent a propositional logic theory, and then recursively define a probability distribution over it. Terminal nodes can be either a literal, $\top$, or $\bot$, while decision (intermediate) nodes are of the form $(p_1 \wedge s_1)\vee \dots \vee (p_k \wedge s_k)$, where the $p_i's$ are called primes and the $s_i's$ subs. The primes form a partition, meaning they are mutually exclusive and their disjunction is valid. Each prime $p_i$ in a decision node  is assigned a non-negative parameter $\theta_i$ such that $\sum_{i=1}^k \theta_i=1$ and $\theta_i=0$ if and only if $s_i=\bot$. Additionally each terminal node corresponding to $\top$ has a parameter $\theta$ such that $0 < \theta < 1$. Using this notation, a PSDD node $n$  defines a distribution over the variables of the \textit{vtree node}   $u$ that it is normalized for, as follows. (The notion of a vtree, defined in \cite{vtree}, is needed to fully define an SDD; they can be obtained directly from data or by compiling domain constaints \cite{Liang2017LearningTS}.) \begin{itemize}
    \item If $n$ is terminal node, and $u$ has variable $X$, then 
    \begin{center}
        \begin{tabular}{l||l|l}

~~~$n$ &  $Pr_n(X)$ & $Pr_n(\neg X)$\\
\hline
~~~$\top$ & $\theta$ & 1 - $\theta$\\
\hline
~~~$\bot$ & 0 & 0\\
\hline
~~~$X$ & 1 & 0\\
\hline
$\neg X$ & 0 & 1\\

\end{tabular}
    \end{center}
\item If $n$ is a decision node and $u$ has left children $\bf X$ and right children $\bf Y$, then $Pr_n ( {\bf x}, {\bf y})= Pr_{p_i} ({\bf x}) \cdot Pr_{s_i} ({\bf y}) \cdot {\theta}_i$, for $i$ where $ {\bf x} \models p_i$, and where $Pr_{p_i}(\cdot),Pr_{s_i}(\cdot)$ denote the distribution of the PSDD nodes corresponding to $p_i,s_i$, respectively.
\end{itemize}

{\bf SPNs.}~~ SPNs are rooted directed graphical models that provide for an efficient way of representing the network polynomial \cite{Darwiche2000ADA} of a BN \cite{6130310}, as a multilinear function $\sum_{\textbf{x}}f(\textbf{x})\prod_{n=1}^{N} \mathbbm{1}_{x_n}$. Here $f(\cdot)$ is the  (possibly unormalized) probability distribution of the BN, $\textbf{x}$ is a vector containing all the variables of the model, i.e.,  $x_1,\cdots,x_N$, the summation is over all possible states, and $\mathbbm{1}_{x_n}$ is the indicator function. An SPN $\cal S$ over Boolean variables $x_1,\cdots,x_N$ has leaves corresponding to indicators $\mathbbm{1}_{x_1},\cdots,\mathbbm{1}_{x_n}$ and $\mathbbm{1}_{\bar x_1},\cdots,\mathbbm{1}_{\bar x_n}$
and whose internal nodes are sums and products.

Any edge exiting a sum node has a non-negative weight assigned to it. The value of a product node is the product of its children, while the value of a sum node is a weighted sum of its children, $\sum_{u_j \in Ch(u_i)} w_{ij}\S_{j}(\textbf{x})$, where $Ch(u_i)$ is the set containing the children of node $u_i$, and $\S_{j}$ is the sub-SPN rooted at node $u_j$. SPNs can represent a wide class of models, including weighted mixtures of univariate distributions; see  \cite{6130310} for discussions. 

{\bf Causality.}~~ 
 We  base our causal analysis on SEMs \cite{Pearl2009CausalII}, which provide an effective way to encode dependencies between variables, as well as allow for queries regarding interventions   and counterfactuals. In this setting  we can represent a set of probabilistic dependencies through a BN, as usual, but on top of that we can also encode the specific mechanism that determines the value of each variable. In this sense, it is more general than just having a BN, since we not only possess a distribution over the variables, but also a (either stochastic or deterministic) set of equations. In what follows we denote by $\V$ the set of variables that are internal to the model, and by $\U$ the exogenous or external variables (that act as random, latent, factors). We use $\pazocal{R}$ to denote the set containing the plausible values of each variable. Every endogenous (internal) variable is assigned an equation determining its value as a function of both its endogenous and exogenous parents in the BN, called structural equation. Finally, in what follows, we make the standard assumption that these BNs do not contain any directed cycle, so they are equivalently referred to as  directed acyclic graphs (DAGs). 
 

\begin{definition}\label{defn:causalmodel}  A causal model $\pazocal{M}$ is a pair $(\pazocal{S},\pazocal{F} )$ where $\pazocal{S}$ is a signature $(\U, \V,\pazocal{R})$ and $\pazocal{F}$ 
is a set of structural equations $ \{\pazocal{F}_V : V \in \V \}.$

\end{definition}

One of the advantages of employing graphical models is that by just utilizing the topology of the graph we can answer probabilistic queries, such as whether two sets of variables are dependent. 

\begin{definition}\label{defn:dsep}
 A directed path is d-separated (blocked) by a set of nodes, \Z, iff one of these hold:
\begin{enumerate}
\item It contains a triple $ X  \rightarrow Z \rightarrow Y$ or $X \leftarrow Z \rightarrow Y$, such that $Z \in \Z$.
\item It contains a triple $X\rightarrow Z \leftarrow Y$, such that neither \( Z \) nor any of its descendants are in \Z.
\end{enumerate}  
\end{definition}

\indent Two sets \X,\Y\ are d-separated by \Z\ if and only if every path between any two nodes $X \in \X$, $Y \in \Y$ is blocked by \Z. It is a well established result that if two nodes are d-separated by a set \Z, then they are conditionally independent (where \Z\ is the conditioning set). As we mentioned earlier, SEMs allow for studying interventional distributions, meaning the distribution of a set of variables, after we force a second set of variables to attain certain values. We denote the distribution of $Y$ after the intervention $X=x$, by $\Pr(Y=y|do(X=x))$ or $\Pr(y|do(x))$. In order to study such probabilistic statements we transform the original DAG corresponding to our model, by deleting all the edges pointing towards \( X \), set \( X \) to \( x \), and then proceed with the analysis. What follows  is an essential graphical tool for deciding under what conditions we can reduce interventional queries to conditional ones. Here, $G_{\overline{x}}$ denotes the graph obtained after deleting all the edges pointing to \( X \), $G_{\underline{x}}$ the one resulting after deleting all the edges emerging from \( X \), and $G_{\overline{x} \underline{z}}$ for deleting both kinds of edges from \( X \).

\begin{definition} \textbf{(Rules of do-Calculus)} Let $\pazocal{G}$ be a DAG corresponding to a SEM and $\Pr(\cdot)$ the probability measure induced by it. If \X, \Y, \Z, \W\ are disjoint sets, then the following hold:
\begin{itemize}
\item[] \textbf{Rule 1:} $\Pr(y| do(x),z,w)$ = $\Pr(y| do(x),w)$ if ${(\Y \independent \Z| \X,\W)}_{G_{\overline{X}}}$.
\item[] \textbf{Rule 2:}  $\Pr(y| do(x),do(z),w)= \Pr(y| do(x),z,w)$ if ${(\Y \independent \Z| \X,\W)}_{G_{\overline{X} \underline{Z}}}$.
\item[] \textbf{Rule 3:} $\Pr(y| do(x),do(z),w)= \Pr(y| do(x),w)$ if ${(\Y \independent \Z| \X,\W)}_{G_{\overline{X} \overline{Z(W)}}}$, where \Z(\W) is the set of \Z-nodes that are not ancestors of any \W-node in $G_{\overline{X}}$.
\end{itemize}
\end{definition}

\section{Main Results}

\subsection{The SPN Case}

As discussed, SPNs are an elegant formalism for capturing weighted mixtures of distribution, and so the expressive power of SPNs and standard BNs has been of considerable interest. The question of how to transform SPNs to BNs and the recoverability of the SPN from the transformation was studied in \cite{Zhao:2015:RSN:3045118.3045132}. 
For space reasons, we cannot provide  too many details on this transformation, but the key idea is to compactly represent the local conditional probability distribution in the corresponding BN by exploiting context specific independence. Intuitively, we create a node for every observable variable, a latent variable for every sub-SPN, and then draw an arrow from each latent variable to the observable variables corresponding to the scope of the sub-SPN. This procedure yields a bipartite graph with arrows stemming only from latent to observable variables.

To our knowledge this is the only way proposed so far to turn an SPN to a BN, and many subsequent papers on SPNs' theoretical properties \cite{Peharz2016OnTL} are similar in thrust. And, as stated previously, the authors of \cite{Zhao:2015:RSN:3045118.3045132} were hopeful about the causal expressiveness of their approach. 

So we will base our analysis on that approach.

We first make the following technical observation about graphs having this topology.

\begin{theorem}
Let $\pazocal{G}$ be the DAG associated with a causal model $\pazocal{M}=(\pazocal{S},\pazocal{F})$. For any set $\X \subseteq \V$ such that no node in \( 
\X \) has an edge coming out of it, the interventional distribution of the remaining variables equals their joint distribution, i.e. $\Pr(u,v_{-x}|do(x))=\Pr(u,v_{-x})$,  where $u \subseteq \U$ and $v_{-x} \subseteq \V \setminus \X$. More specifically, we have that the rest of the remaining variables are unaffected by the intervention, i.e. $\Pr(v_{-x}|do(x))=\Pr(v_{-x})$.
\end{theorem} 
\begin{proof} 
Using the 3rd rule of Pearl's do-calculus, it suffices to show that ${(\X \independent \U \cup (\V \setminus \X))}_{G_{\overline{X}}}$. By assumption, no edges emanate from nodes in $\X$, which implies that each of them will be isolated in $G_{\overline{X}}$, so the desired independence holds, meaning that $\Pr(u,v_{-x}|do(x))=\Pr(u,v_{-x})$. In addition, we have that $\Pr(v_{-x}|do(x))= \sum_{U} \Pr(u,v_{-x}|do(x)) = \sum_{U} \Pr(u,v_{-x}) = \Pr(v_{-x})$. \qed \vspace{10pt}

\end{proof}
Unfortunately, since the BN stemming from the algorithm in \cite{Zhao:2015:RSN:3045118.3045132} has no edge coming out of an observable variable, we get the following:

\begin{theorem}\label{th1 SPN}
The BN, $\cal G$, that results after transforming an SPN using the procedure described in \cite{Zhao:2015:RSN:3045118.3045132} satisfies the property $\Pr(v_{-x}|do(x))=\Pr(v_{-x})$, for any $\X \subseteq \V$.
\end{theorem}

So this result answers that the method proposed in \cite{Zhao:2015:RSN:3045118.3045132} for producing a BN is not useful for causal inference tasks, since we cannot really study interventional distributions utilizing it.

\indent What are the reasons for this limitation? As has been noted in previous work \cite{Peharz2016OnTL,6130310}, sum nodes in SPNs can be interpreted as marginalized, latent, variables, whose values correspond to the children of the sum node. Thus, when an SPN is turned into a BN all of the variables within the scope of a sum node are treated as children of a latent variable. This leads to every probabilistic dependency being attributed to an unobserved confounder, and there is no edge between the SPN variables. Thus, it is reasonable that any intervention on a subset of the observable variables would not affect the rest, because the mechanism encoded in the graph tells that no variable has any causal effect on the others.

\indent A special class of SPNs, referred to as selective SPNs were introduced recently \cite{Peharz2014LearningSS}. They impose determinism in that only one of the children of a sum node can be true for any given variable assignment. Interestingly, even this stipulation does not remedy the problem, since the discussion in \cite{Peharz2016OnTL} makes clear the resulting BN would still have no edges between the SPN variables. Consequently, we get:
\begin{theorem}\label{th2 SPN}
The BN, $\cal G$, that results after transforming a selective SPN using the procedure described in \cite{Peharz2016OnTL}  satisfies the property $\Pr(v_{-x}|do(x))=\Pr(v_{-x})$, for any $\X \subseteq \V$.
\end{theorem}
We suspect that to get rid of this limitation SPNs should be augmented in a way that captures the functional dependency between the variables in the scope of a sum node. Another strategy perhaps is to enable a more expressive way to represent and reason about probabilistic dependencies between the variables, although it is not immediate how this could be made possible.

\subsection{The PSDD Case}

Interestingly, the situation turns out to be slightly better for PSDDs. The intuitive reason for that is because of the dependency that can be established between a node in the PSDD and its children. More precisely, consider that each node $n$ in a PSDD \cite{kisa2014probabilistic} has a support -- the set of assignments it assigns a positive probability -- which is related to the support of its children. This set is called the $\textit{base}$ of $n$ and is denoted by $[n]$. It can also be defined as a logical formula: if $n$ is a decision node $(p_1 \wedge s_1)\vee \dots \vee (p_k \wedge s_k)$, then $[n]= \bigvee_{i=1,\dots,k} [p_i] \wedge [s_i]$. Since $p_i$'s form a partition, their corresponding bases are disjoint, as well, so a decision node can be seen as deciding between different possible worlds, based on which prime base was satisfied. Since the prime bases of a node form a partition we can apply the law of total probability and Proposition 1 from \cite{liang2017learning} to get that $Pr_n(\textbf{x},\textbf{y}) = \sum_{i=1,\dots,k} Pr_n(\textbf{x},\textbf{y},[p_i]) = \sum_{i=1,\dots,k} Pr_n(\textbf{x},\textbf{y}|[p_i])\cdot Pr_n([p_i]) = \sum_{i=1,\dots,k} Pr_n(\textbf{x}|[p_i])\cdot Pr_n(\textbf{y}|[p_i]) \cdot Pr_n([p_i])$. Combining this expression with the semantics provided in Proposition 1 in \cite{liang2017learning} and Theorem 2 in \cite{kisa2014probabilistic}, as well as the fact that under any given assignment the only non-zero term of the form $Pr_n(\textbf{x}|[p_i])$ is the one for which $\textbf{x} \models p_i$, we see that the probability of a node is not a mixture over its children (as in SPNs). Indeed, the distribution of  decision node is understood very differently. In fact, we can also see that PSDD nodes do not condition on a latent variable, but on their prime bases instead, which do not depend on unobserved quantities.

\indent  Our work builds on this observation and the fact that by construction PSDDs are probabilistic extensions of SDDs, which, in turn, denote a propositional formula. Basically, we use that formula to create an augmented set of variables, not just the original ones the PSDD used for training, in such a way so the PSDD distribution and the BN one agree on the original variables. It is worth noting that the resulting BN is also equipped with a set of equations that determine the value of the children as functions of their parents, so we end up having a SEM. Below we present the procedure to construct this SEM, where the input propositional formula is the one represented by the trained PSDD.

\begin{algorithm}[h]
\KwIn{A formula $\phi=c_1\vee \dots \vee c_n$, where $c_i=p_i\wedge s_i$, over propositional variables $x_1,\dots x_k$}
\KwOut{A SEM model with the augmented variables}

\nl  Create a variable corresponding to the whole expression, $v_0=\phi$\;
\nl  Create a variable, $v_i$, for each $c_i$\;
\nl  Create an arrow from $v_i$ to $v_0$, $i=1,\dots,n$\;
\nl  For each $c_i=p_i\wedge s_i$ create a variable $v_{i}^{p}$ for $p_i$ and a variable $v_{i}^{s}$ for $s_i$\;
\nl  Create an arrow from $v_{i}^j$ to $v_i$, for $j \in \{p,s\}$\;
\nl  Repeat this process recursively, until the original variables are reached\; 
\nl When this procedure is over, create a hidden variable and connect it with each one of the original variables with an arrow pointing at them.
    \caption{{\bf Expression to SEM} \label{Algorithm}}

\end{algorithm}

\noindent \textbf{About the hidden variable: }The latent variable, $\textbf{H}$, is the only component of the graph that is purely stochastic, and we motivate its need here. 

Note that any instantiation of it is enough to determine all the other variables in the model. Conversely, each probabilistic query about any of the rest of the variables can be reduced into another query relying solely on \textbf{H} (since there is no other source of randomness in the model). Its dimension is equal to the number of the original variables in the PSDD and its distribution is equal to the PSDD distribution of the original variables.  Denoting by $\Pr_{SEM}(\cdot)$ the probability measure over the DAG's variables and by $\Pr_{PSDD}(\cdot)$ the PSDD distribution over the original variables, we set these two measures to satisfy the following condition $\Pr_{SEM}(H_1,H_2,\dots,H_n)=\Pr_{PSDD}(X_1,X_2,\dots,X_n)$. Suppose the PSDD is  comprised of $n$ variables, $X_1,X_2,\dots,X_n$, then $\textbf{H}=(H_1,H_2,\dots,H_n)$. The structural equations connecting them are: \(
	X_1=H_1, X_2=H_2,\dots,X_n=H_n. 
\)

Looking at these equations we see that: $\Pr_{SEM}(X_1,X_2,\dots,X_n)=\Pr_{SEM}(H_1,H_2,\dots,H_n)=\Pr_{PSDD}(X_1,X_2,\dots,X_n).$
This remark assures us about the consistency between the PSDD and the SEM distribution of the original variables. We would also like to note that although $\textbf{H}$ is introduced as a vector, it could be rewritten as a simple categorical variable with an exponential number of states, each one corresponding to a different configuration of the original variables. We present this result in a more formal way, using the vectorized version of $\textbf{H}$. 

\begin{theorem}
\label{psddprop}
Let \textbf{P} be a PSDD over variables $X_1,X_2,\dots,X_n$ and let $\pazocal{G}$ be the DAG resulting from Algorithm 1. The distribution of $X_1,X_2,\dots,X_n$, induced by $\pazocal{G}$ is equal to their PSDD distribution, meaning that $\Pr_{SEM}(X_1,X_2\dots,X_n)=\Pr_{PSDD}(X_1,X_2,\dots,X_n)$.
\end{theorem}

Interestingly, the SEM obtained from a PSDD in this manner has the same limitations as identified for SPNs when intervening on the original variables: 
\begin{theorem}
The SEM, $\cal S$, that results after applying Algorithm 1 to a PSDD compiled formula satisfies the property $\Pr(v_{-x}|do(x))=\Pr(v_{-x})$, where $X$ is any subset of the original variables, and $v_{-x}$ denotes the rest of the original variables.
\end{theorem}

\begin{proof}
We are going to use the 3rd rule of Pearl's do-calculus, so it is enough to show that for any path bewtween the original variables is blocked. Let $X$ be the variable we intervene on and let $Y$ be any of the rest of the original variables. We have to show that ${(X \independent Y)}_{G_{\overline{X}}}$. By construction, since the BN is created using Algorithm 1, there are no edges between the original variables. Furthermore, no original variable is a descendant of another one, since the only parent of an original variable is the latent variable. This means that, in $G_{\overline{X}}$, all the paths connecting $X$ and $Y$ contain v-structures, so they are blocked and the 3rd rule is satisfied. Since $Y$ was chosen at random, we can generalize this result for arbitrary subsets of the original variables, concluding the proof. \qed  
\end{proof}

However, when intervening on the augmented variables, we are able to enable non-trivial (but also non-standard) causal reasoning, a point we return to shortly.

\indent Moreover $\textbf{H}$ serves another  purpose as we will shortly discuss. Using the BN from Algorithm 1 without including the hidden variable,  it is not difficult to see that the original PSDD variables are independent, since all the paths connecting them are blocked by {\it v-structures}, meaning that   (2) in Definition \ref{defn:dsep} is satisfied, with $\Z=\emptyset$. On the other hand, it is not necessarily the case that the PSDD distribution encodes such properties about the variables, so there is a chance that the BN distribution enforces independences that do not agree with the PSDD one, rendering the DAG unfaithful \cite{Pearl:2009:CMR:1642718}. By including the hidden variable we eliminate this behaviour, but we introduce a new property, the other extreme, that all of the variables are dependent. This might also not be the actual case either, but we think that it is safer to assume dependency among the variables, rather than independency, which is a fairly strong assumption. A better way to address this behaviour would be to utilize the PSDD distribution and some independency tests in order to decide the subsets of dependent variables, and then use as many hidden variables as the dependent subsets, so we explicitly encode only the dependencies that are implied by the PSDD distribution. (Incidentally, such tests are  used when learning SPNs \cite{SPN_structure_learning}.)  
In this work we are mostly interested in introducing the connection between BNs and PSDDs, so we leave this  for future research. 

 We should also note that for any node $X$ in the graph resulting from Algorithm 1, denoting the set of its parents as $\textbf{PA}_X$, we have:
\begin{equation*}
\Pr(X=1| \textbf{PA}_X) = \begin{cases}
             1  &  \text{if  assignments in } \textbf{PA}_X \text{ render } X=1 \\
             0  &  \text{otherwise }
       \end{cases} 
\end{equation*}
Building on top of this remark, the distribution of $X$ given the specification of any partial subset of its parents $\V \subset \textbf{PA}_X$ is as follows:
\begin{equation*}
\Pr(X=1| \V) = \begin{cases}
             1  &   \text{if  assignments in } \V \text{ render } X=1 \\
             0  &  \text{if  assignments in } \V \text{ render } X=0 \\
            \Pr(X|_{\V}=1)  &  \text{otherwise}
       \end{cases} 
\end{equation*}
where $X|_{\V}$ denotes the formula that results from $X$ after substituting  the assignments from $\V$ in it.  Finally, 
the marginal distribution of $X$, for example,  can be computed by using the PSDD.

\begin{figure}[t]
  \centering
    \includegraphics[scale=0.5]{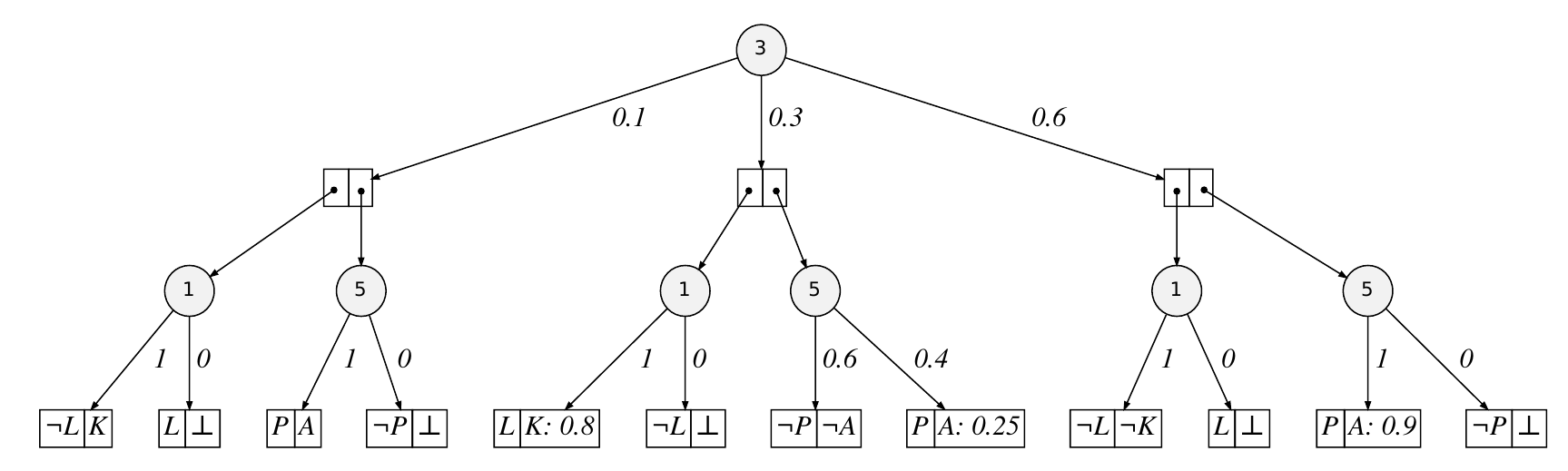}
  \caption{A PSDD over variables $A,L,K,P$}
\end{figure}

{\bf Example:}~~{We will give an example of how to construct a SEM model using a PSDD. We start by studying    the PSDD in Figure 1. This is the PSDD corresponding to a problem considered in \cite{kisa2014probabilistic}. The setting is that there is a department having four courses: Logic (L), Knowledge Representation (K), Probability (P), and Artificial Intelligence (A). Students must enroll to them, but at the same time they have to obey the following constraints:  \( P\lor L \), \( A\Rightarrow P \), \( K\Rightarrow (A\lor L) \) (where implication means if they enroll in LHS, then they must enroll in RHS). The objective is to learn the joint distribution of $L,K,P,A$ using a dataset of student enrollments and the above constraints. The authors utilize PSDDs to perform this task and the resulting model can be seen in Figure 1. }

Starting from the bottom of Figure 1 and  moving towards the root, we see that it corresponds to the following propositional formula:
\begin{align*}
&(((\neg L \wedge K) \vee (L \wedge \bot))) \wedge ((P \wedge A) \vee (\neg P \wedge \bot)))  \\
&\vee (((L \wedge K)\vee (\neg L \wedge \top)) \wedge ((\neg P \wedge \neg A) \vee (P \wedge A))) \\
& \vee (((\neg L \wedge \neg K) \vee ( L \wedge \bot)) \wedge ((P \wedge A) \vee (\neg P \wedge \bot)))
\end{align*}
This is the raw form of the  formula, so some terms are tautologically false. Rewriting the above expression after eliminating inconsistencies yields the following:
\begin{align*}\label{eq:simp form}
&((\neg L \wedge K) \wedge (P \wedge A)) \vee ((L \wedge K) \wedge ((\neg P \wedge \neg A) \vee (P \wedge A))) \\ 
& \vee ((\neg L \wedge \neg K) \wedge (P \wedge A)) \tag{$\star$}
\end{align*}
Algorithm 1 takes  \eqref{eq:simp form}   as input and constructs a SEM model, as follows: The first thing is to create a node corresponding to the whole expression. Then, since \eqref{eq:simp form} is composed of three disjunctions, we make three new variables, one for each of them, and draw arrows from them pointing to the first variable. We continue this procedure recursively; so for example, the term $(\neg L \wedge K) \wedge (P \wedge A)$ is made from two formulas  that are connected with a conjunction, so we create two new nodes, one for $\neg L \wedge K$ and one for $P \wedge A$, and draw arrows from them towards the node representing their conjunction. Now we have reached the point where the formulas under consideration are just conjunctions of literals, so if we look at $\neg L \wedge K$, we make a node for the variable $L$ (although it is  $\neg L$ that is part of the formula) and one for $K$. We repeat the above procedure until we go through all the formulas and in the end we create an additional latent variable that is a parent of all the original PSDD variables, here $A,L,K,P$. The resulting BN can be seen in Figure 2 (Left). It is also worth noting that since we create at most two new nodes for any disjunction or conjunction, the size of the BN is linear in the their number. Furthermore, looking at the procedure described above, we see that Algorithm 1 is not directly applicable to SPNs, since sums and products are between distributions, while in PSDDs, conjunctions are disjunctions are between variables, which is exactly what Algorithm 1 exploits in order to construct the resulting SEM.

\begin{figure}[t]
  \centering
  \begin{tabular}{cc}
         \includegraphics[scale=0.5]{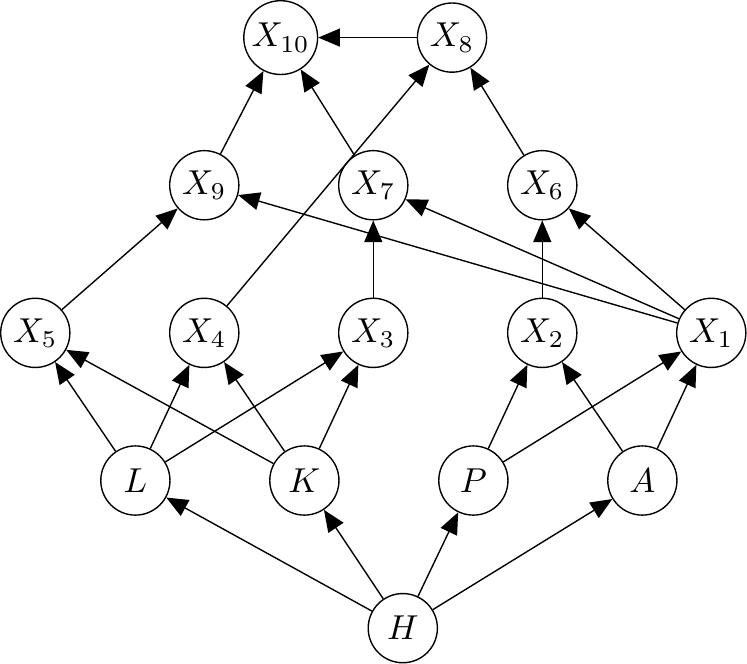}
  ~~~~~
          \includegraphics[scale=0.5]{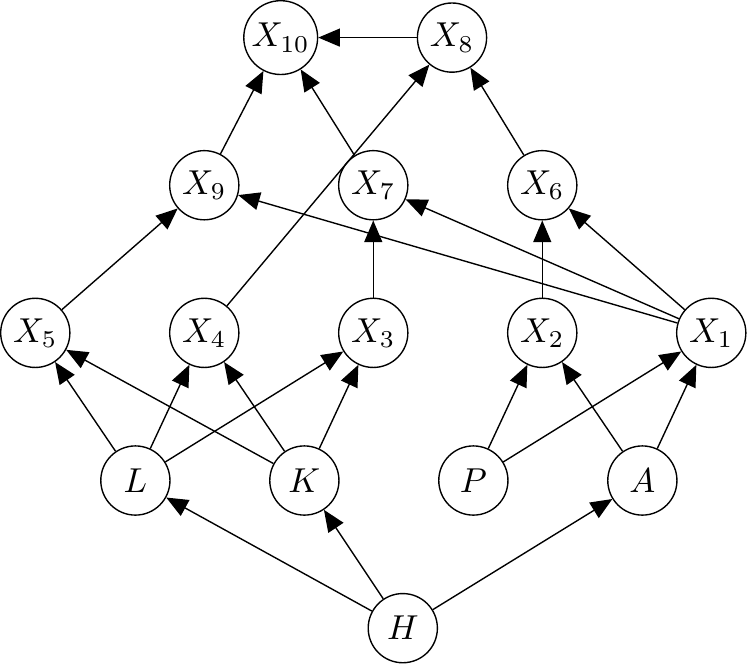}

  \end{tabular}
  \caption{Left: The augmented BN model corresponding to the PSDD of figure 1. Right: The same model, after intervening on $P$.}

\end{figure}

We have kept the names of the original variables the same and have named the rest as $X_1,\dots , X_{10}$. In addition, the latent variable is a vector \( (H_1, \ldots, H_4) \) of 4  random variables, since the PSDD was over four variables. By construction, it is now apparent that the structural equations of this BN are:
\begin{align*}
&A=H_1, L=H_2, K=H_3, P=H_4, X_1 = P \wedge A, X_2 = \neg P \wedge \neg A, \\
& X_3 = \neg L \wedge K, X_4 = L \wedge K,  X_5 = \neg L \wedge \neg K, X_6 = X_1 \vee X_2,  \\
& X_7 = X_1 \wedge X_3,X_8 = X_4 \wedge X_6,X_9 = X_1 \wedge X_5,X_{10} = X_7 \vee X_8 \vee X_9
\end{align*}

We can immediately see that 
an intervention on one of the augmented variables will result in a non-trivial interventional distribution. We expand on this point in the next section.  

\subsection{Interventions}

In this section we are going to state a result connecting the interventional to the observational distribution. The idea behind is that, since there is no noise in the graph, every augmented variable will be a deterministic function of its parents. In turn, this reduces the problem of estimating interventional queries to a simpler one, which is, essentially, a problem of counting all the possible assignments of the parents of the intervened variable, that then leads to the intervened variable getting a corresponding  value. This transforms the question from a statement of the form ``how probable it is to observe $Y=y$, given an intervention $X=x$" to a statement of the form ``how probable it is to observe $Y=y$ and $X=x$, simultaneously". Formally, we have the following:
\begin{theorem}
Let $\pazocal{G}$ be the DAG resulting from Algorithm 1, and let $Y,X$ be two augmented variables. Then for any intervention $do(X=x)$, we have that $\Pr(Y=y|do(X=x)) = \Pr(Y=y,X=x)$. 
\end{theorem}

\begin{proof}
We are going to base our proof on the back-door criterion \cite{Pearl:2009:CMR:1642718}, adjusting for the parents of $X$, $X_1,\cdots,X_N$. By doing that, we can rewrite our expression as follows:
\begin{align*}
&\Pr(Y=y|do(X=x)) = \\
&\sum_{x_1,\cdots,x_N} \Pr(Y=y|X=x,X_1=x_1,\cdots,X_N=x_N)\cdot \Pr(X_1=x_1,\cdots,X_N=x_N) \\ &= \sum_{x_1,\cdots,x_N : X=x} \Pr(Y=y|X_1=x_1,\cdots,X_N=x_N)\cdot \Pr(X_1=x_1,\cdots,X_N=x_N) \\ &= \sum_{x_1,\cdots,x_N : X=x} \Pr(Y=y,X_1=x_1,\cdots,X_N=x_N) =  \Pr(Y=y, X=x)
\end{align*}
The first equality is due to the back-door criterion, the second one is because since $X$ is a deterministic function of its parents, all the assignments of $X_1,\cdots,X_N$ that do not result in $X=x$, make the term $\Pr(Y=y|X=x,X_1=x_1,\cdots,X_N=x_N)$ equal to zero. On the other hand, all the assignments that result in $X=x$ lead to $\Pr(Y=y|X=x,X_1=x_1,\cdots,X_N=x_N) = \Pr(Y=y|X_1=x_1,\cdots,X_N=x_N)$, since now the condition $X=x$ is redundant. \qed
\end{proof}

The above formula clearly gives rise to a non-trivial distribution, although its usefulness and the overall utility of performing causal analysis based on it, should probably be assessed depending on the application. In the next section we demonstrate how counterfactual queries can be estimated using the output of Algorithm 1.

\subsection{Counterfactuals}

\indent In this section we will examine if it is possible to use the BN from Algorithm 1 in order to compute counterfactual quantities. We will mostly investigate counterfactuals conditioned on some evidence, which is equivalent to computing probabilistic statements of the form $\Pr(Y=y|do(X=x),E=e)$. These statements can be handled using the following major result \cite{Pearl:2009:CMR:1642718}: 

\begin{theorem} 
Let $\pazocal{M}$ be a causal model and $\Pr(\cdot)$ a probability measure over the variables in $\U$. The counterfactual probability $\Pr(Y=y|do(X=x),E=e)$, meaning \enquote{Had $X$ been $x$ then $Y$ would have been $y$, given evidence $e$}, can be computed as follows:
\begin{itemize}
\item \textbf{Abduction: } Update the distribution $\Pr(\U)$ by incorporating the evidence, to obtain $\Pr(\U|e)$.
\item \textbf{Action: } Construct the graph that results from the intervention $do(X=x)$.
\item \textbf{Prediction: } Use the probability measure and the graph, from the previous steps, to compute the probability of $Y=y$.
\end{itemize}
\end{theorem}

\begin{table} \small \centering
    \begin{tabular}{l|l|l|l||l}

$L$ &  $K$ & $P$ & $A$ & PSDD distribution\\
\hline
0 & 0 & 1 & 0 & 6.0\%\\
\hline
0 & 0 & 1 & 1 & 54.0\%\\
\hline
0 & 1 & 1 & 1 & 10.0\%\\
\hline
1 & 0 & 0 & 0 & 3.6\%\\
\hline
1 & 0 & 1 & 0 & 1.8\%\\
\hline
1 & 0 & 1 & 1 & 0.6\%\\
\hline
1 & 0 & 0 & 0 & 3.6\%\\
\hline
1 & 1 & 0 & 0 & 14.4\%\\
\hline
1 & 1 & 1 & 0 & 7.2\%\\
\hline
1 & 1 & 1 & 1 & 2.4\%\\
\end{tabular}
    \caption{Distribution of $L,K,P,A$ from \cite{kisa2014probabilistic}}

    \label{tab:my_label}
\end{table}

\indent Since our model is deterministic, we do not have to do a lot of probabilistic calculations, but mostly we are going to manipulate logical expressions. We will go on with our working example to demonstrate how  we could study counterfactuals and their properties. The question of interest is the following: Supposing we have observed that $X_1=0$, what is the probability it would have been equal to 1, had P been equal to 1? At this point we would like to emphasize that an intervention on one of the augmented variables corresponds to multiple interventions on the original ones. For example, suppose that later on we decide to intervene on $X_1$ and force it to become equal to zero. In turn, this would mean that we force $A \wedge P$ to become zero. We  notice that this outcome can be achieved by several assignments on these two variables, namely $A=P=0$ , $A=0 \text{ and } P=1$, and  $A=1 \text{ and } P=0 $. This means that a single intervention on $X_1$ induced three interventions on $A$ and $P$ simultaneously. {We would also like to emphasize that although $X_1$ belongs to the augmented set of variables, it still has an interpretation relating it to the original variables, as it is the case with any of the augmented variables. In this case, $X_1$ just represents the event of taking both courses, $A$ and $P$.}

Formally, we ask for the probability of the following expression $\Pr(X_1=1|do(P=1),X_1=0)$. The first step is to update the distribution of our exogenous variables (in our case, this is \textbf{H}) conditioning on the evidence $X_1=0$. As we have already discussed, $X_1=0\Leftrightarrow (P \wedge A) =0$, so  $P$ or $A$ is equal to zero. This means the updated distribution should assign zero probability to the case of $P$ and $A$ being true at the same time. Since this is the only fact we can recover from the conditioning observation, the posterior and the prior distributions should agree on all other cases. Thus, we end up with $\Pr(H_1,H_2,H_3,H_4| X_1=0)$ being obtained as: 
\begin{equation*}
 \begin{cases}
              \frac{\Pr(H_1,H_2,H_3,H_4,X_1=0)}{\Pr(X_1=0)} = \frac{\Pr(H_1,H_2,H_3,H_4)}{\Pr(X_1=0)} & \text{if }H_1=0 \text{ or }H_4=0 \\
             0  & \text{otherwise }
       \end{cases} 
\end{equation*}
The upper branch equality follows from Bayes' Theorem and the fact that $(H_1=0 \vee H_4=0) \Leftrightarrow X_1=0$.
Next, we construct the graph corresponding to the world where we intervene on $P$ and force it to be true, which is shown in Figure 2 (Right). Now we update the structural equations, by substituting $P=1$ to all the equations. Since we are not going to make use of all of them in this particular example, we will write down only the first few.  
\begin{align*}
& A=H_1, L=H_2, K=H_3, P=1, X_1 = A, X_2 = 0
\end{align*}
Now we are ready to perform all the desired calculations,in our case the probability of $X_1=1$ in the causal graph of Figure 2 (Right). We proceed as follows:
\begin{align*}
\allowbreak
\Pr(X_1=1)&=\Pr(A=1)=\Pr(H_1=1)=\Pr(H_1=1,H_4=0)\\
&=\sum_{H_2}\sum_{H_3} \frac{\Pr(H_1=1,H_2,H_3,H_4=0)}{\Pr(X_1=0)}
\end{align*}
We immediately see that all of the needed probabilistic quantities can be calculated right away using the PSDD and the correspondence between $H_1,H_2,H_3,H_4$ and $A,P,K,L$. \smallskip 

{\indent We could also ask more complex counterfactual queries as well. This time we will include the actual numeric values, so that we can compare the various distribution of the variable of interest. The data is taken from \cite{kisa2014probabilistic} is shown here, but the full calculations can be found in the supplementary material. The question this time is supposing we have witnessed that $X_9=P \wedge A \wedge \neg L \wedge \neg K=0$, meaning there is a student not satisfying the property \enquote{he/she has taken both $A$ and $P$, while not taking neither $L$ or $K$}: what is the probability of him/her satisfying this property, had $A$ been equal to $1$. We repeat the same steps as before, to obtain the probability $\Pr(X_9=1|do(A=1),X_9=0)$, which turns out to be equal to ${\Pr(H_1=0,H_2=0,H_3=0,H_4=1)}/{0.46} = {0.06}/{0.46} \approx 0.13$. We compare the resulting counterfactual distribution to the conditional $\Pr(X_9|A=1)$ and the plain marginal $\Pr(X_9)$. The results can be seen in Figure \ref{distr}. It is evident that the counterfactual distribution is vastly different from the others, expanding the semantics of PSDDs in a non-trivial way.  }

\begin{figure}[t]
  \centering
    \includegraphics[scale=0.62]{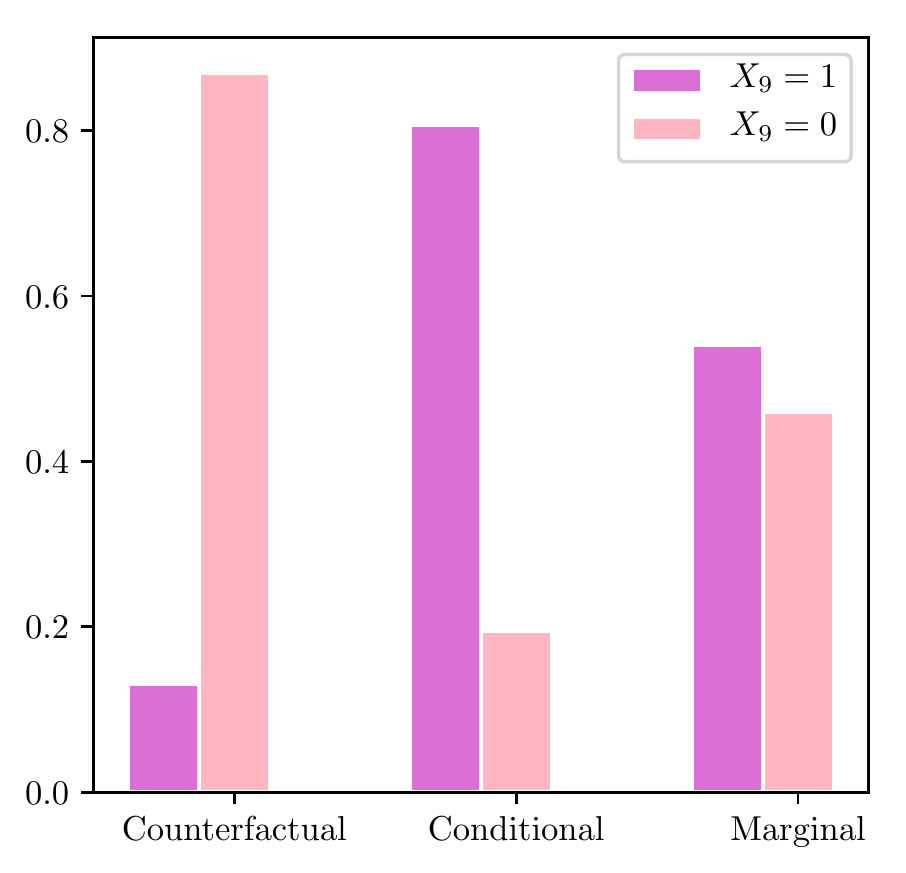}
  \caption{Comparison of distributions}
  \label{distr}
\end{figure}

\section{Discussion and Conclusions}

Tractable models are attractive in offering polynomial time inference capabilities, and hence are gaining in popularity. The theoretical properties of such models have received considerable attention recently. 
The question of whether these models can also be useful for causal reasoning was studied in this work, and we showed that the results are mostly of a negative nature. For SPNs, we showed that we cannot really study interventional distributions. For PSDDs, we motivated a way to construct a SEM from a trained PSDD. We showed that when intervening on the original variables, the situation is once again uninteresting, but when non-trivial properties emerge when augmented variables are considered. While this does provide a causal semantics for PSDDs, we observe the causal graph is very unusual in lacking noise. So, the overall usefulness of these class of tractable models is questionable for causal reasoning.\textit{ We would like to reiterate that the thrust of this contribution assumes that the only information we have is the probabilistic circuit. Clearly if we had the original BN in hand, we would perform causal reasoning directly on that BN.} However, starting from the circuit, we show that going to the BN loses information about the underline mechanisms that the variables interact with each other, as it is evident when using SPNs.
For PSDDs, although the outcome of the analysis is about the same, the problem is of a different nature, and it is mostly due to absence of latent factors. In many cases, in causal modeling, we tend to include some latent variables, in order to account for unobserved background factors, but it is unclear why one should do this for PSDDs. From a causal viewpoint, SPNs and PSDDs also seem to be on the opposite sides of the spectrum, one former attributing everything to latent factors, while the latter attributing nothing to them. To recap, SPN sum nodes define weighted mixtures over their children, while PSDD decision nodes are propositional expressions over them. This difference lies in the core of the nature of our results.

\indent We think there are many interesting directions for the future. For example, given our last observation about latent factors, are there tractable models that enable causal graphs with lie somewhere on the middle ground wrt causal graphs? Current structure learning algorithms for tractable models also do not attempt to capture the underlying causal process. In that regard, the quality of the causal graph obtained in Algorithm 1 is only going to be as good as the quality of the PSDD. We think there are three ways to improve this situation, both under the assumption that we are in possession of prior knowledge in terms of certain dependencies and independencies. Firstly, a brute force (and very likely inefficient) structure learner would first build a PSDD, use Algorithm 1 to recover all the dependencies and interactions between variables and test whether our prior knowledge is in agreement with what the PSDD has learned. An insufficient model would then be discarded by means of a suitable evaluation metric. Secondly, it is shown in \cite{liang2017learning} that the training of PSDDs can be subjected to logical prior knowledge. It may be possible to extend that approach, in that we learn PSDDs that are also subjected to independency constraints expressed as probabilistic prior knowledge. Thirdly, and perhaps most significantly, investigating whether ideas from the existing literature on learning causal relations (e.g.,\cite{DBLP:conf/icml/GhassamiSKB18}) can be imported to tractable learners is a worthwhile question. Of course, such an endeavor would be most useful if we discover ways to augment SPNs and PSDDs in some (clever) way that goes beyond trivial and/or deterministic reasoning. That would be perhaps the main open challenge resulting from our work.\\

\bibliographystyle{abbrv}
\bibliography{paper}

\clearpage

\end{document}